\documentclass[letterpaper,11pt]{article}
\usepackage{amssymb,amsbsy,amsmath,amscd,amsthm,amsfonts}
\usepackage{cite}
\usepackage{enumerate}
\usepackage{dsfont}
\usepackage{graphicx}
\usepackage{longtable}
\usepackage{tablefootnote}
\usepackage{standalone}
\usepackage{fancyhdr}
\usepackage{float}
\usepackage{caption}
\usepackage{subcaption}
\usepackage{url}
\usepackage[margin=1.35in]{geometry}
\usepackage[numbers]{natbib}

\RequirePackage[colorlinks = true,
            linkcolor = black,
            urlcolor  = black,
            citecolor = black,
            anchorcolor = black]{hyperref}

\newcommand{\calF}{{\mathcal{F}}}

\newcommand{\calU}{{\mathcal{U}}}

\newcommand{\calX}{{\mathcal{X}}}

\newtheorem{theorem}{Theorem}

\newtheorem{lemma}{Lemma}
\newtheorem{remark}{Remark}

\newtheorem{corollary}{Corollary}

\usepackage{prettyref}
\newrefformat{eq}{(\ref{#1})}
\newrefformat{thm}{Theorem~\ref{#1}}
\newrefformat{sec}{Section~\ref{#1}}
\newrefformat{algo}{Algorithm~\ref{#1}}
\newrefformat{fig}{Fig.~\ref{#1}}
\newrefformat{tab}{Table~\ref{#1}}
\newrefformat{rmk}{Remark~\ref{#1}}
\newrefformat{def}{Definition~\ref{#1}}
\newrefformat{cor}{Corollary~\ref{#1}}
\newrefformat{lmm}{Lemma~\ref{#1}}
\newrefformat{prop}{Proposition~\ref{#1}}
\newrefformat{app}{Appendix~\ref{#1}}
\newrefformat{ex}{Example~\ref{#1}}
\newrefformat{ass}{Assumption~\ref{#1}}

\usepackage[]{algorithm2e}
\usepackage{authblk}
\title{Complexity, Statistical Risk, and Metric Entropy of Deep Nets Using Total Path Variation}
\author[1]{Andrew R. Barron\thanks{andrew.barron@yale.edu}}
\author[2]{Jason M. Klusowski\thanks{jason.klusowski@rutgers.edu}}
\affil[1]{Department of Statistics and Data Science, Yale University}
\affil[2]{Department of Statistics, Rutgers University \textbf{--} New Brunswick}

\begin{document}

\maketitle

\begin{abstract}%
For any ReLU network there is a representation in which the sum of the absolute values of the weights into each node is exactly $1$, and the input layer variables are multiplied by a value $V$ coinciding with the total variation of the path weights.  Implications are given for Gaussian complexity, Rademacher complexity, statistical risk, and metric entropy, all of which are shown to be proportional to $V$. There is no dependence on the number of nodes per layer, except for the number of inputs $d$.  For estimation with sub-Gaussian noise, the mean square generalization error bounds that can be obtained are of order $V \sqrt{L + \log d}/\sqrt{n}$, where $L$ is the number of layers and $n$ is the sample size.
%
%For examination of the statistical risk of multi-layer neural network estimation available methods include Rademacher complexity, Gaussian complexity, metric entropy and the minimum-description-length principle.  We add to the understanding of the relationship between these by explicating a metric entropy bound from the risk of estimators by information-theoretic methods.  A fresh look at contraction properties of complexity measures is developed that allows to produce improved bounds on the risk of deep ReLU nets. 
\end{abstract}


\vspace{-0.1cm}
\section{Introduction}

We motivate examination of the tools of complexity and metric entropy by considering a  statistical learning network task. From data, estimate a function in a class $\calF$ of functions defined by successive layers of compositions of Lipschitz nonlinearities and $\ell_1$ bounded linear combinations.  The training data  $(X_i,Y_i)$ are independent with distribution $P_{X,Y}$, for observation indices $i=1,\ldots,n$, where the response is $Y_i = f^*(X_i)+\varepsilon_i$, the noise $\varepsilon_i$ is Gaussian (or sub-Gaussian), the inputs $X_i$ are $d$-dimensional, residing in the bounded input domain $[-1,1]^d$ with arbitrary design distribution $P_X$, and there is a known bound on the target function  $|f^*(x)|\!\le\! B$ for all $x$ in that input domain.   The estimator $\hat f_n$ may be defined to minimize the empirical squared error $\sum_{i=1}^n (Y_i\!-\!f(X_i))^2$, or a penalized version there-of, minimizing over choices of $f$ in $\calF$. The statistical loss is the expected square over choices of possible future inputs distributed according to the same $P_X$, yielding the $L_2(P_X)$ squared norm $\|f-f^*\|^2$. % = \int (f(x)-f^*(x))^2 P_X(dx)$.  
The statistical risk (generalization error) is the expected statistical loss $r(\hat f_n,f) = E[\|\hat f_n -f^*\|^2]$.
% which entails both the expectation $E$ with respect to the distribution of the training data and the expectation with respect to future inputs per the definition of the squared norm.

For examination of statistical risk properties in function estimation (such as multi-layer neural nets) standard methods include Gaussian and Rademacher complexity \citep{Bartlett+Mendelson,Shalev-Shwartz+,Neyshabur+,Bartlett+,E+,Golowich+}
%(Bartlett and Mendelson (2003), Shalev-Shwartz and Ben-David (2014), Neyshabur et al (2015), Bartlett, Foster and Telgarsky (2017), E, Ma and Wu (2018), Golowich, Rakhlin and Shamir (2018)) 
and the minimum-description-length principle armed with metric entropy and related cover arguments \citep{Barron+Barron,Barron1991,Barron1994,Barron+1999,Barron+2008, Yang+,Huang+,Klusowski+,Barron+2018}.
%(Barron et al (1988,1990,1994,2007,2008,2018), Yang et al (1998), Huang et al (2008), Klusowski et al (2017)).  

Each of these methods produce risk bounds possibly depending on the size of the model and a norm on the parameters and inversely on the sample size.  Results for multi-layer networks as in \citep{Bartlett+,Klusowski+,E+,Golowich+,Barron+2018}
%Bartlett et al (2017), Klusowski et al (2017), E et al (2018), Golowich et al (2018) and Barron et al (2018) 
reveal risk bounds which involve the ratio $(\log d)/n$, where the bound depends on either the square root or the cube root of this ratio.  If suitable norm controls are in effect, the dependence on the number of inputs via the logarithm allows for very large numbers of variables and parameters compared to the sample size. 

Methods based on Rademacher complexity and its variants (as in \citep{Bartlett+,E+,Golowich+})
%Bartlett et al (2017), E et al (2018), and Golowich et al (2018))
are useful for multi-layer function classes defined via compositions of Lipschitz functions and $\ell_1$ bounded linear combinations. Incuring a factor of $2$ cost for each of the $L$ layers, there are bounds of the form $2^L  ((\log d)/n)^{1/2}$
%Neyshabur et al (2015) and Bartlett et al (2017), 
times bounds on the products of norms on the weight matrices as in \citep{Neyshabur+,Bartlett+}. More recently \citep{Golowich+} uses a variant of Rademacher complexity to reduce the dependence on $L$ to a low order polynomial, with no cost from the numbers of units per layer. 

Metric entropy methods, with covers established by sampling arguments, also yield a low order dependence on $L$ as in \citep{Barron+2018}. The bounds there take the form $\bar V L^{3/2} ((\log d)/n)^{1/2}$, where $\bar V$ is a composite norm on the network weights. The composite variation is an average of the all input and output subnetwork variations. % as explained in  \citep{Barron+2018}.  
An even more sensible norm on the weights would be the total network variation $V$ which is an $\ell_1$ path norm. It is the sum across all paths through the network of the products of the absolute values of the weights along the paths \citep{Neyshabur+}.  Precursors to these multi-layer bounds are the one hidden-layer bounds in \citep{Klusowski+} and \citep{E+}. 
 
Here the risk bounds are improved to be of the form $V ((L+\log d)/n)^{1/2}$, where $V$ is the network variation.  This is related to the work of \citep{Golowich+} who obtained bounds based on exponential variants of Rademacher complexity.  Our main improvement is the recognition that we can use the total variation $V$ of the path weights rather that the product of matrix norms used in \citep{Golowich+}.  Other generalizations here are to allow unbounded sub-Gaussian noise, as well as demonstration of metric entropy and Gaussian complexity bounds as well as Rademacher complexity. 
%A key is careful developments of contraction properties of complexities applied to a contractive representation of multilayer ReLU networks.

Impressions arise as to the relative merits of metric entropy and Gaussian/Rademacher complexity methods.  It might seem that the Gaussian/Rademacher complexity approach is superior. Though in principle it should not be possible to out-perform metric entropy methods because of the known relationships to minimax risk as developed in \citep{Yang+} and references sited there-in. 

We build on the information-theoretic methods based on Fano's inequality to produce a clean bound on the metric entropy from ingredients of statistical risk bounds. One implication is that Rademacher/Gaussian complexity measures, like metric entropy, can be used as a penalty ingredient in a minimum empirical risk criterion, as a valid way to endow the criterion with rigorous total description-length interpretation. 

%
%So the Rademacher complexity is indeed a flexible tool, sufficiently malleable to produce the desired conclusions.  Does that mean that it is intrinsically superior to metric entropy and minimum description length methods?  At first one might gain that impression.  Here though, we use information theory techniques to show that Rademacher complexity upper bounds do indeed imply upper bounds on metric entropy, beyound what was previously developed.  Indeed the Rademacher complexity is shown for our setting to exceed $\epsilon$ times the square root of the $\epsilon$ entropy. Accordingly Rademacher complexities (more precisely the square of such divided by a suitable $\epsilon$) can be used as upper bounds on log cardinalities of optimal $\epsilon$ covers and hence used as information-theoretic description-length penalities.

%Ledoux and Talagrand 
The relationship between metric entropy, Gaussian complexity and Rademacher complexity is studied in \citep{Ledoux+Talagrand}, Chapters 3 and 4. The Gaussian complexity is between a Sudakov lower bound and a Dudley integral upper bound, both involving the metric entropy. \citep{Ledoux+Talagrand} addresses whether similar relationships holds for Rademacher complexity in Prop. (4.13) and Cor. (4.14). The lower bound is hampered by requirement of a small diameter for the set being covered or by the presence of an additional factor, logarithmic in the ratio of the dimension and the complexity.  Our analysis traps the statistical risk between expressions involving metric entropy and expressions involving Gaussian complexity.

%Our improved Sudakov inequality is based on the following simple logic.  The Rademacher complexity indexed by $n$ provides an upper bound on the minimax risk of function estimation based on samples of size $n$. In certain sampling settings the minimax risk is known to be lower-bounded by (indeed of the same order as) a metric entropy $\log N_\epsilon$ divided by sample size $n$ at choice of $\epsilon$ for which this ratio matches $\epsilon^2$.  Particulars of this metric entropy lower bound on risk are worked out in Yang et al (1998) by an information-theoretic use of total Kullback divergence and Fano's inequality.  Applying this reasoning to a Cesaro average of the minimum empirical risk estimator projected to an optimal $\epsilon$ net produces the desired inequality.

Section \ref{sec:deep} gives the deep network framework and exposes the fundamental role of the total variation $V$. In Section \ref{sec:complexity} we adapt comparison inequalities for Rademacher complexity and Gaussian complexity to see the implication of the deep network framework for these complexity measures, showing that they are proportional to $V$.
%for function classes defined via layers of composition of convex combinations and Lipschitz nonlinearities. 
 We use these to produce desired inequalities for multilayer networks.  
These considerations permit risk bounds taking the simple form $V \, ((L+\log d)/n)^{1/2}$, with very modest dependence of the number of layers $L$ and the number of inputs $d$.

Section \ref{sec:entropy} relates the metric entropy for subsets $A$ of $R^n$ to risk bounds for estimating $a$ in $A$ from $Y\sim$ Normal($a,\sigma^2I$). 
%Gaussian location $Y_i = a_i + \varepsilon_i$ with $a= (a_i)_{i=1}^n$ in $A$.  
The least squares estimator produces upper bounds on the metric entropy of $A$.   %Such bounds are available based on complexity measures which bound the risk.  
The noise is Gaussian, so that the total Kullback divergence matches the squared Euclidean distance.  Nevertheless, once the inequality is proven, metric entropy penalties and risk bounds can be extracted from complexity measures for estimation with various other error distributions, even when it is not Gaussian.
These considerations are applicable to the function estimation problem with the $a_i$ replaced by $f(X_i)$. 

Finally, Section \ref{sec:risk} uses the tools developed in the previous sections to give risk bounds for deep ReLU networks. Both constrained least squares estimators and penalized least squares estimators (for adaptation to lack of prior knowledge of $V$) are discussed.  Along the way, the complexity-based risk bound of constrained least squares provides a metric entropy evaluation for the deep nets. This knowledge enables determination of a penalty as a specific multiple of $V$ that does indeed provide a criterion with a minimum description-length interpretation and resultant adaptive risk bound.

\section {Deep Nets} \label{sec:deep}

Let $F_{L}$ be the class of $L$ layer deep nets with positive part activation function $\phi(z)= (z)_+$, also called a lower-rectified linear unit (ReLU).  
 %That is, each coordinate of the input vector $x$ in $R^d$ is assumed to be bounded by $1$.  
We follow \citep{Barron+2018} for a helpful organization of the weights of the network.  
%Peculiarly, this notation 
It uses $\ell=1$ to denote the outermost layer and $\ell=L$ to denote the innermost layer.
%, for reasons that we will make clear.

The original variables $x_1,x_2,\ldots,x_d$, each take values in the interval $[-1,1]$ and are inputs to the innermost layer,
%Let $d$ be the number of original variables on the first layer, each taking values in the interval $[-1,1]$
%The $d$ inputs to the innermost layer are the original variables,
 with one of them locked at the value $1$ (for off-sets of the subsequent units), along with their opposite-sign duplicates
$-x_1,-x_2,\ldots,-x_d$, producing an input vector $z_L$ of size $d_L=2d$. 
%An instance of this input vector is denoted $x$. 
%We start creating a vector $z_L$ with coordinates $z_{j_L}=x_{j_L}$. 
Then $d_{L-1}$ linear combinations of these are taken followed by the positive part, producing a vector $z_{L-1}$ with coordinates $z_{j_{L-1}} = \phi( \sum_{j_L} w_{j_{L-1},j_L} z_{j_L})$. We provide freedom of signs of the linear combination by the inclusion of the opposite-sign duplicates of the original variables.
%, that is, $z_{d+j_L}= x_{d+j_L}=  -x_{j_L}$ for $j_L=1,2,\ldots,d$, with $d_L=2d$.  
This allows restricting the $w_{j_{L-1},j_L}$ to be non-negative.
% at the cost of a doubling of the range of the index $j_L$.  
Likewise, in anticipation of the need to combine these with sign freedom, we double the number of units created using $z_{d_{L-1}+j_{L-1}} = - z_{j_{L-1}}$, that is, using the activation functions which are minus the positive parts.

At the output of this first layer, to have one of the $d_{L-1}$ units we create be locked at $+1$, it is obtained by having its input weights be concentrated on the original variable locked at $1$.  Its duplicate will be concentrated at $-1$.

Continuing in this way, at layer $\ell$ we have inputs $z_{j_\ell}$, including opposite-sign duplicates, and outputs $z_{j_{\ell-1}}= \phi_{j_{\ell-1}} (\sum_{j_\ell} w_{j_{\ell-1},j_{\ell} z_{j_\ell}})$.  For the first half of indices the $\phi_{j_{\ell-1}}$ is the positive part and for the second half it is minus the positive part.  Again units locked at $+1$ (and its duplicate locked at $-1$) are obtained at layer $\ell\!-\!1$ by concentrating their input weights at the the locked unit on the more inward layer $\ell$.
%by such units at layer $\ell$ with input weights concentrated at the locked unit at the previous layer.

Finally at the output we have $f(w,x) = \sum_{j_1} w_{j_1} z_{j_1}$ or $\phi_{out}(\sum_{j_1} w_{j_1} z_{j_1})$, where $\phi_{out}$ is an optional %odd-symmetric 
Lipschitz $1$ activation function at the final output, which can be useful to clip the output range to a specified $[-B,B]$.
Again the $z_{j_1}$ appear with duplicates of opposite sign. 

Altogether, in this way, we implement the freedom of signed weights while constraining all the $w_{j_1}$, $w_{j_1,j_2}$, $\ldots,$ $w_{j_{L-1},j_L}$ to be non-negative.  We see that an arbitrary depth $L$ ReLU network, with $d_1$, $d_2$, \ldots $d_L$ nodes on layers $1$ through $L$, is expressed as such a network with all weights non-negative and $d_\ell$ replaced by $2(d_\ell+1)$ to account for offsets and signed duplicates.
  
%
%There is an option at the output layer to incorporate a final arbitrary Lipschitz activation function $\phi_{out}$ with a Lipschitz constant of $1$, whereas all the internal units use plus or minus ReLU activation functions.  One natural choice for $\phi_{out}(z)$ is the two-sided rectifier, equal to $z$ for $-B \le z \le B$ and clipped to the value $B$ for $z>B$ and to the value $-B$ for $z<-B$.  This is a Lipschitz $1$ function which allows to keep the output within a possibly known range of the target response.  It is odd-symmetric, so its inclusion preserves closure of our deep net function classes under sign change.  If no nonlinearity is desired at the output one may set $\phi_{out}(z)=z$.

As also explained in \citep{Barron+2018} the use of the (plus or minus) positive parts, allows repeated use of their positive homogeneity, to move weights inward to the innermost layer, reorganize them, and then parcel them back out to the respective layers in normalized form, without changing the function represented. The only stipulation is that when moving weights inward we need to keep track of their path indices.  Indeed, $f(w,x)$ is
$$\phi_{out} ( \sum_{j_1} w_{j_1} \phi_{j_1} (\sum_{j_2} w_{j_1,j_2} \cdots \phi_{j_{L-1}}( \sum_{j_L=1}^{2d} w_{j_{L-1},j_L} x_{j_L}) \cdots ))$$
which may be written as
$$  \phi_{out} ( \sum_{j_1} \phi_{j_1} (\sum_{j_2}  \cdots \phi_{j_{L-1}}( \sum_{j_L=1}^{2d} w_{j_1,j_2,\ldots,j_L} x_{j_L}) \cdots ))$$
where the path weights $w_{j_1,j_2,\ldots,j_L}$ are  $w_{j_1}  w_{j_1,j_2} \cdots w_{j_{L-1},j_L}$.
Let $V=V(w)$ be the variation (also called the $\ell_1$ norm) of these path weights $V=  \sum_{j_1,j_2,\ldots,j_L}w_{j_1,j_2,\ldots,j_L}$.
%As explained in \citep{Barron+2018} it is also called the total variation of the network.
Moreover, let $a = w/V$ provide probabilities assigned to the paths obtained from the normalized weights.
This $a_{j_1,j_2,\ldots,j_L}$ has a Markov factorization
$a_{j_1} a_{j_2|j_1} \cdots a_{j_L|j_{L-1}}$. 
%These normalized weights $a$ are also denoted $p_{j_1,\ldots,j_L}$ with representation $p_{j_1} p_{j_2|j_1} \cdots p_{j_L|j_{L-1}}$ below.
%Now, in seeing the interpretation as a probability distribution on the path indices, starting at the outermost $j_1$ and working our way successively inward to $j_2$ and so on, ending at the innermost $j_L$, we can indeed see reason for this peculiar reversed order of representation of the layer indices.

Care is needed in interpreting this Markov representation of the weights $a_{j_\ell|j_{\ell-1}}$.  Fixing $j_{\ell-1}$, they are not proportional to $w_{j_{\ell-1},j_\ell}$ in general.  But rather they are proportional to  $w_{j_{\ell-1},j_\ell} V_{j_\ell}$ where $V_{j_\ell} = \sum_{j_{\ell+1},\ldots,j_L} w_{j_\ell,j_{\ell+1}} w_{j_{\ell+1},j_{\ell+2}} \cdots w_{j_{L-1},j_L}$
is the subnetwork variation of the network that ends at node $j_\ell$ at layer $\ell$. 

Now with this Markov representation we have $f(w,x)=f(aV,x)=f(a,Vx)$ 
%given by
%$$\phi_0 ( \sum_{j_1} \phi_{j_1} (\sum_{j_2}  \cdots \phi_{j_{L-1}}( \sum_{j_L=1}^{2d} a_{j_1} a_{j_2|j_1} \cdots a_{j_L|j_{L-1}} V x_{j_L}) \cdots ))$$
and we can distribute the weights back out in the form
$$\phi_{out} (  \sum_{j_1 }a_{j_1} \phi_{j_1} (\sum_{j_2} a_{j_2|j_1} \cdots \phi_{j_{L-1}}( \sum_{j_L=1}^{2d}   a_{j_L|j_{L-1}} V x_{j_L}) \cdots )),$$
where again we have used the positive homogeneity of the activation functions $+(z)_+$ and $-(z)_+$.

This proves the following theorem about the representation of arbitrary deep ReLU networks.

%\noindent {\bf Theorem 1:} 
\begin{theorem} \label{thm:relu}
Any deep network with ReLU activation can be written as a network with probabilistic weights summing to $1$ over the inputs to any unit, with the original inputs %and their sign-flopped duplicates 
all multiplied by the overall variation $V$.
\end{theorem}

For the following section, we let $\calF = \calF_{L,V}$ be the class of $L$ layer deep nets with variation not more than $V$.

\section{Complexities Satisfying Network Contraction} \label{sec:complexity}

For a function class $\calF$ constructed from multiple layer compositions of convex combinations and Lipschitz $1$ functions, we investigate complexity measures that enjoy contraction under such operations. We build upon the ideas of \citep{Golowich+}, which permit one to avoid the exponential dependence on depth in the complexity bounds (Rademacher and Gaussian) that are used to bound the generalization error and statistical risk. The Rademacher case is covered by their paper \citep{Golowich+}, except that we extract what can be an improved constant $V$ (also interpretable as the $ \ell^1 $-norm of the product of the weight matrices) instead of the product of matrix norms of the weight matrices.
%We build upon ideas of Ledoux and Talagrand (1991, Chapter 4) in providing these properties.

For a symmetric perturbation distribution $\mu$ on the line, we let $\xi_1, \xi_2, \ldots, \xi_n$ be i.i.d. according to $\mu$.  Then for a subset $A$ of $R_n$, the $\mu$-Complexity of $A$ denoted $C(A)$ is given by $E \left[\sup_{a \in A}  ( \sum_{i=1}^n \xi_i a_i)_+ \right].$
%$$C(A) \; = \; E \left[\sup_{a \in A}  ( \sum_{i=1}^n \xi_i a_i)_+ \right].$$
For symmetric sets $A$, for which $-A=A$, the supremum inside the expectation is always positive so it is then the same as
$E \left[\sup_{a \in A}  (\sum_{i=1}^n \xi_i a_i)\right]$.  When $\mu$ is symmetric Bernoulli or Gaussian, respectively, then $C(A)$ is the Rademacher Complexity or Gaussian Complexity.

Let $\psi(z)$ be any positive, increasing and convex function on $z \ge 0$.
%, with the property that the symmetric function $\bar \psi (z) = (\psi (z)+\psi(-z))/2$ also remains increasing, as well as convex, for positive $z$.  
A prime example is $\psi(z)=e^{\lambda z}$ with $\lambda > 0$.  Closely related is $G(z)=\psi(z_+)$ which is a nondecreasing, convex function of $z$ on the line, where $z_+ = \max\{z,0\}$.
%It has convex $\bar \psi(z)$ with derivative $\lambda (e^{\lambda z} - e^{-\lambda z})$ which is indeed positive for $z > 0$.
%Depending on the choice of $\psi$ upperbounds on the $\mu-$Complexity (via Jensen's inequality) are obtained from the following variant. 
Fix such a convex $\psi$. The $\mu,\psi$-Complexity of $A$, denoted $C(A)=C_\psi (A)=C_{n,\mu,\psi}(A)$, is given by
$$\psi(C (A)) \; = \; M(A) \; = \; E \left[\sup_{a \in A}  \psi(( \sum_{i=1}^n \xi_i a_i)_+)\right].$$
Note that $C_\psi (A)$ is obtained by applying the inverse of $\psi$ to the expected supremum.  In particular for $\psi(z) = e^{\lambda z}$, we have $C_\psi (A) = \frac{1}{\lambda} \log M(A)$.
% where $M(A)$ is $E \left[\sup_{a \in A} \psi(( \sum_{i=1}^n \xi_i a_i)_+)\right]$.
When $\mu$ is symmetric Bernoulli or Gaussian, respectively, then $C(A)$ is a natural variant of the Rademacher Complexity or Gaussian Complexity. For increasing convex $\psi$ these variants provide (via Jensen's inequality) potentially useful upper bounds on the complexities in their original forms, especially in the context of contractive compositions.
% (also called Gaussian width in probabilistic geometry). 

%By the symmetry of $\bar \psi(z)$ it is the same as $\bar \psi(|z|)$. Accordingly
%$$M(A)\; = \; E\left[\sup_{a \in A} \bar \psi( |\sum_{i=1}^n \xi_i a_i|)\right],$$
%and by the monotonicity of $\bar \psi(z)$ for positive $z$, we have
%$$M(A)\; = \; E\left[\bar \psi (\sup_{a \in A} |\sum_{i=1}^n \xi_i a_i|)\right].$$
%The set $A$ is said to be symmetric when $a$ in $A$ implies $-a$ is in $A$. For symmetric $A$, for any achieved value of 
%$\sum_{i=1}^n \xi_i a_i$, the same value of opposite sign is achievable by using $-a$.  This means that for each $\xi_1,\ldots,\xi_n$, the supremum of achieved values of the even $\bar \psi (\sum_{i=1}^n \xi_i a_i)$ is equally well achieved at $-a$. Accordingly, for symmetric $A$, by the monotonicity of $\bar \psi(z)$ on each side of $0$, we also have
%$$M(A)\; = \; E\left[\bar \psi (\sup_{a \in A} \sum_{i=1}^n \xi_i a_i)\right].$$

An important setting is when $A =\{(f(X_1),\ldots,f(X_n)):f \in \calF\}$ is the set  $\calF_{X^n}$ of achievable function values in $\calF$ restricted to the data $X^n$.  The resulting $C(\calF_{X^n})$ is called the empirical $\mu,\psi$-Complexity of $\calF$. In our deep net example, this $A=\calF_{X^n}$ is symmetric when there is there is either no nonlinear final output activation function $\phi_{out}$ or when it is odd-symmetric, as in the case of the function that clips the output to the range $[-B,B]$.
%, also called the trace of $\calF$ on $X^n$.
%, and it is symmetric when the function class is symmetric.
For this deep net function class, we will have unform control of the complexity $C(\calF_{X^n})$ for all data with $X_i$ in $[-1,1]^n$.  These can be used to provide upper bounds on risk.  

Let $\phi \circ \calF = \calF_\phi = \{\phi(f(\cdot)): f \in \calF\}$ be constructed by composition of functions in $\calF$ with a contraction $\phi$, that is, a Lipschitz $1$ function $\phi$ with $\phi(0)=0$.  As we shall see it provides no increase in complexity.  The same is true with the taking of a convex hull. Signed convex hulls are nearly contractive for these complexities, introducing only a factor of $2$. In deep net literature these are familiar facts for Rademacher complexity and these conclusions are also available for Gaussian complexity. 

Given a class of functions $\calF$ let $\calF' = \hbox{conv}\{\pm \phi \circ \calF \}$ be the class of functions with one more layer.  That is,
$$\hbox{conv}\{\pm \phi \circ \calF \} = \{\sum_j c_j \phi( f_j (\cdot) )\; :\; f_j \in \calF, \sum_j |c_j| \le 1\}.$$ 
Likewise  we form $\phi \circ A  = \{(\phi(a_1),\ldots,\phi(a_n)) :  a \in A\}$ for any $A$ in $R^n$
and let $A' = \hbox{conv}\{\pm \phi \circ A \}$ be given by 
$$\{(\sum_j c_j \phi( a_{j,1} ),\ldots,\sum_j c_j \phi( a_{j,n} ) ) \; :\; a_j \in A , \sum_j |c_j | \le 1\}$$ 
where $a_j = (a_{j,1},\ldots,a_{j,n})$.

%\par\vspace{0.2cm}
%\noindent {\bf Lemma 3:} 

\begin{lemma}[Contraction Lemma] \label{lmm:contract}
For any contraction $\phi$, for any increasing convex $\psi$, and for $\mu$ symmetric Bernoulli or Gaussian, $M(\phi \circ A) \le M(A)$. Likewise $M(\hbox{conv} A) \le M(A)$ and
$$M(\hbox{conv}\{\pm \phi \circ A \}) \le 2M(A).$$
\end{lemma}
 
%\noindent {\bf Proof:} 
\begin{proof} In the definition of the set $A'$ write the $c_j$ as $\varsigma_j p_j$ where $\varsigma_j\in\{-1,+1\}$ is the sign and $p_j$ is the magnitude. 
%Write the set $A'$ as 
%$$\left\{(\sum_j p_j \phi_j( a_{j,1} ),\ldots,\sum_j p_j \phi_j( a_{j,n} ) ) \; :\; a \in A , \sum_j p_j \le 1\right\}$$ 
%where the $\phi_j$ are selections from the set $\{\phi,-\phi\}$. 
Now
$$M(A') = E[\sup_{p_j,\varsigma_j,a_j: j=1,2,\ldots} \psi( (\sum_i \xi_i \sum_j p_j \varsigma_j \phi(a_{j,i}))_+)],$$
where the sumpremum is for all choices of $a_j$ from $A$ along with signs $\varsigma_j$ and  weights $p_j$ which sum to not more than $1$.
Exchanging the sums it is
$$E[\sup_{p_j,\varsigma_j,a_j: j=1,2,\ldots} \psi( (\sum_j p_j \varsigma_j \sum_i \xi_i \phi(a_{j,i}))_+)].$$
Use the monotonicity of $\psi((z)_+)$, 
%its monotonicity with increasing magnitude of $z$, the symmetry obtained by the $\pm1$ values of $s_j$, 
and the fact that an average over $j$ is less than the maximum to get this is less than 
$$E[\sup_{\varsigma_j,a_j: j=1,2,\ldots} \psi( (\max_j \varsigma_j \sum_i \xi_i \phi(a_{j,i}))_+)].$$
By the monotonicity again this is
%$$E[\sup_{a \in A, \varsigma \in \{-1,1\}} \psi( (\varsigma \sum_i \xi_i \phi(a_i))_+)].$$
%This is 
$M(\pm \phi \circ A) = E[\sup_{a \in A, \varsigma \in \{-1,1\}} G( \varsigma \sum_i \xi_i \phi(a_i))]$,
which is at most
$$
E[\sup_{a \in A}G(\sum_i \xi_i \phi(a_i))] + E[\sup_{a \in A}G(-\sum_i \xi_i \phi(a_i))] = 2E[\sup_{a \in A}G(\sum_i \xi_i \phi(a_i))].
$$
If the $ \xi_i $ are Rademacher, one uses \citep{Ledoux+Talagrand}, equation (4.20) to address the effect of the contraction and bound $ E[\sup_{a \in A}G(\sum_i \xi_i \phi(a_i))] $ by $ E[\sup_{a \in A}G(\sum_i \xi_i a_i)] $, as in \citep{Golowich+}. 

If the $ \xi_i $ are Gaussian with variance $ \sigma^2 $, we use inequality (13) of \citep{Vitale} with $ m_i = 0 $, which provides a variant of the Fernique inequality for Gaussian contraction, analogous to the cited inequality 
%\citep{Ledoux+Talagrand}, equation (4.20) 
for Rademakers. Let $Z_a = \sum_i \xi_i a_i$ and $\tilde Z_a = \sum_i \xi_i \phi(a_i)$ be mean zero Gaussian processes indexed by $a$ in $A$.  For this variant of the Fernique inequality in \citep{Vitale},  it suffices that $E[(\tilde Z_a - \tilde Z_b)^2]$ is not more than $E[(Z_a\!-\!Z_b)^2]$ and that $E[\tilde Z_a^2]$ is not more than $E[Z_a^2]$, for $a$ and $b$ in $A$.  By the form of $Z_a$ and $\tilde Z_a$ these quantities are $\sigma^2$ times the squared norms $\sum_i (\phi(a_i)-\phi(b_i))^2$ and $\sum_i (a_i\!-\!b_i)^2$, as well as $\sum_i (\phi(a_i))^2$ and $\sum_i (a_i)^2$, respectively. The desired inequalities hold (term-by-term) by the assumed contractive properties of $\phi$. 
%It suffices that $ E[|\sum_i \xi_i \phi(a_i)|^2] = \sigma^2\sum_i \phi^2(a_i) \leq \sigma^2\sum_i a^2_i = E[|\sum_i \xi_i a_i|^2] $ and $ E[|\sum_i \xi_i[\phi(a_i)-\phi(b_i)]|^2] = \sigma^2\sum_i |\phi(a_i)-\phi(b_i)|^2 \leq \sigma^2\sum_i |a_i-b_i|^2 = E[|\sum_i \xi_i[a_i-b_i]|^2] $ to imply $ E[\sup_{a \in A}G(\sum_i \xi_i \phi(a_i))] \leq E[\sup_{a \in A}G(\sum_i \xi_i a_i)] $. The inequality between the second moments is a consequence of the fact that $ \phi $ is a contraction, i.e., $ |\phi(z)| \leq |z| $.
This completes the proof of Lemma \ref{lmm:contract}.
\end{proof}

\par\vspace{0.2cm}
Let $A^L = ((A')') \ldots )'$ be the subset of $R^n$ obtained by $L$ layers of the operation  $( )' = \hbox{conv}\{\pm \phi \circ (\cdot)\}$ starting from a set $A$. Likewise let $\calF^L$ be the class of functions obtained by $L$ layers of such operation starting from a function class $\calF$, with corresponding traces $A^L= \calF_{X^n}^L$, in $R^n$ from restriction to data $X^n$.  In the Rademacher case the following conclusion is from \citep{Golowich+}.

%\noindent{\bf Corollary 3:}
\begin{corollary}[Multiple layer contraction]
For any set $A$, the Rademacher and Gaussian complexity of $A^L$ satisfies $M(A^L)\le 2^LM(A)$. %and hence for function classes $M(\calF^L) \le 2^LM(\calF)$.  
In particular, using $\psi_{\lambda} (z) = e^{\lambda z}$, we have that $C(A^L) \le C_{\psi_{\lambda}} (A^L)$ is not more than 
$$\frac{1}{\lambda} L \log (2)  + C_{\psi_{\lambda}} (A).$$ 
\end{corollary}

Let \emph{init} $= \{x_1,x_2,\ldots,x_d,-x_1,-x_2,\ldots,-x_d\}$ be the collection of signed coordinate functions of $x=(x_1,\ldots,x_d)$ in $[-1,1]^d$, and let $\calF_{0,V} = V$\emph{init} be such scaled by $V$. Then using the contraction $\phi(z)=(z)_+$ the class of functions $(\calF_{0,V})^L$ is the set of $L$ layer networks described above with a linear unit on the output layer and with no restriction on the number of units per layer past the first input layer. Moreover $\calF_{L,V} = \phi_{out} \circ (\calF_{0,V})^L$ is the set of such networks with a final Lipschitz function $\phi_{out}$ on the output layer.

With $n$ observed input vector instances $X_i$ in $[-1,1]^d$ for $i=1,2,\ldots,n$, the restriction $A= \emph{init}_{X^n}$ is
$$\{\pm (X_{1,j},X_{2,j},\ldots,X_{n,j}) \; : \; 1\le j \le d\}$$ 
and the restriction of $\calF_{0,V}$ to $X^n$ is this input set multiplied by $V$ of cardinality $2d$ in $R^n$. The following bounds hold uniformly over such possible data $X^n$.
%Then the objects $A^L = \hbox{init}_{X^n}^L$ and $\phi_0 \circ A^L = (\calF_{L,V})_{X^n}$ are sets of vectors in $R^n$ traced by depth $L$ ReLU networks evaluated at the data (again with linear or Lipschitz activations on the last layer, respectively).  

%\noindent{\bf Corollary 4:} 
\begin{corollary}[Complexity bounds for ReLU networks] \label{cor:complexity}
The Rademacher and Gaussian complexities (with $\sigma=1$) of the deep net classes $\calF_{L,V}$ are not more than the complexity using $\psi_\lambda (z)= e^{\lambda z}$  at suitable $\lambda$, which is not more than the complexity 
$$
C(\calF_{L,V}) =  V\sqrt{2n(L\log(2)+\log (2d))}.
$$
%of the finite set $\calF_{0,V}$ of cardinality $2d$. 
This bound is independent of %the number of layers and 
the number of units on each hidden layer. 
%The bound becomes 
%$$
%V\sigma \sqrt{2n(L\log(2)+\log (2d))}.
%$$
%for Gaussian complexity when the $ \xi_i $ have variance $ \sigma^2 $.
\end{corollary}

%\noindent{\bf Proof:}  
\begin{proof} The multiple layer contraction bound reduces the task to evaluation of complexity of the finite set $\calF_{0,V}$ at the input layer, restricted to $X^n$.  This set is of cardinality $2d$ and scale $V$ so
%The first claim is by repeated application of the contraction lemma, as it holds for any positive, increasing, convex $\psi(z)$ on $z>0$. The bound on Rademacher complexity using the exponential $\psi$ is by Jensen's inequality. The 
%bounds on the complexity of the finite set $\calF_{0,V}$ of cardinality $2d$ and scale $V$ 
we proceed as in Massart's Lemma (e.g., Shalev-Shwartz and Ben-David, Lemma 26.8) with the choice there of $\lambda = \sqrt{2\log(2^L2d)/(V^2 n)}$. The Gaussian complexity proof is essentially the same using $\lambda = \sqrt{2\log(2^L2d)/(\sigma^2 V^2 n)}$ and the resulting Gaussian complexity bound is multiplied by $\sigma$ when $\sigma$ is not equal to $1$.
\end{proof}

%\vspace{-0.1cm}
\section{Metric entropy} \label{sec:entropy}

For the Gaussian location problem the full data is $Y=(Y_1,\ldots,Y_n)$ with $Y_m\sim$Normal($a_m,\sigma^2$), independently for $m=1,2,\ldots,n$, and we consider mean vectors $a$ living in a set $A$. A batch estimator $\hat a$ is a function of $Y$ in which each $\hat a_m$ may depend on all the coordinates of $Y$, and a predictive or online estimator is one in which each coordinate $\hat a_m$ only depends on preceding  $Y_1,\ldots,Y_{m-1}$, for each $m$ from $1$ to $n$. The risk of an estimator $\hat a$ is set to be $E[\|\hat a\! \!- a\|^2]$ where $\|b\!-\!a\|^2 = (1/n) \sum_{i=1}^{n} (b_i\!-\!a_i)^2$. 
%is the squared Euclidean metric divided by $n$. 
Let $r_n(A)$ and $r_n^*(A)$, respectively, be upper bounds on the risk of an estimator and a predictive estimator that hold uniformly over $a$ in $A$. Implicitly these depend on $\sigma$ as well as $n$ and $A$.
% and when we want to make that explicit we write $r_n(A,\sigma^2)$ and $r_n^*(A,\sigma^2)$. 
The infimum of $r_n(A)$ over choices of estimators is of course the minimax risk. Let $\log N(\epsilon,A)$ be the Kolmogorov $\epsilon$ entropy (the metric entropy) of the set $A$ with respect to the Euclidean metric, where $N(\epsilon,A)$ is the size of the maximal $\epsilon$-packing of $A$, the largest size subset $\tilde A$ of $A$ such that the distance between members of $\tilde A$ is greater than $\epsilon$.
We use log base $e$.

Risks and metric entropy have the following relationship.

%\noindent {\bf Lemma 1}:
\begin{lemma} \label{lmm:riskmetric}
For all $\epsilon > 0$,
$$r_n(A) \ge \frac{\epsilon^2}{4} \left( 1- \frac{n \,r_n^*(A)/(2\sigma^2) + \log 2}{\log N(\epsilon,A)}\right).$$
Equivalently, for $\epsilon^2 > 4 r_n(A)$,
$$\log N(\epsilon,A) \le \frac{n \, r_n^*(A)/(2\sigma^2) + \log 2} {1 - 4 r_n(A)/\epsilon^2}.$$
In particular, for $\epsilon^2 \ge  8 r_n(A)$,
$$\log N(\epsilon,A) \le n \, r_n^*(A)/\sigma^2 + 2 \log 2.
$$
\end{lemma}

%\noindent {\bf Proof}: 
\begin{proof} Fix an $\epsilon > 0$. Let $\tilde a$ be the projection of the estimator $\hat a$ to the packing set $\tilde A$, meaning that each outcome of the estimator $\hat a$ is replaced by the closest member of $\tilde A$. Then for any $a$ in $\tilde A$, by the triangle inequality, the $\| a - \hat a\|$ is at least $(\epsilon/2) 1_{\{a \neq \tilde a\}}$.  So, as in %Yang et al 
\citep{Yang+}, by Fano's inequality, 
$$r_n(A) \ge (\epsilon/2)^2 \left ( 1- \frac{I(A;Y) + \log 2}{\log N(\epsilon,A)}\right).$$
Where $I(A; Y)$ is the mutual information between the parameter vector and the data $Y$, when the parameter is endowed with the uniform prior distribution on $\tilde A$.  This mutual information is the average, with respect to this prior distribution for the parameter $a$, of the Kullback divergence between the conditional distribution of $Y$ given $a$ and the unconditional distribution for $Y$ that minimizes this Bayes average.  Accordingly, for any joint density $q(Y)$ that does not depend on $a$, we have
$$I(A;Y) \le \max_{a \in A} D(P_{Y|a}\|Q_Y).$$
So far this is akin to the development in \citep{Yang+}. Now arrange a different choice of the joint density $q$ from what is there.  Specifically build it from the conditionals  $q(y_m|y_1,\ldots,y_{m-1})$ that are normal with mean $\hat a_m^*$ and variance $\sigma^2$, where $\hat a^*$ is a predictive estimator with risk bound $r_n^*(A)$. Its conditional Kullback divergence from the Normal($a_m,\sigma^2$) is $(a_m\!-\! \hat a_m^*)^2/2\sigma^2$.  By the chain rule for relative entropy, $D(P_{Y|a}\|Q_Y)= \sum_{m=1}^n E(a_m\!-\! \hat a_m^*)^2/2\sigma^2$ which is $n E\|a\!-\!\hat a^*\|^2/2\sigma^2$, not more than $ n \,r_n^*(A)/2\sigma^2$.  Plugging this bound into the Fano inequality establishes the result. 
% Rearranging it establishes the equivalent claim. 
% This completes the proof of Lemma 1.
\end{proof}

%\noindent {\bf Remarks}:  
%\begin{remark}
%\indent
%\begin{enumerate}
%\item 
This proof bound improves on alternative approaches to the use of Fano's inequality, as in \citep{Birge} and \citep{Tsybakov}, wherein $Q_Y$ is taken to be a  normal $P_{Y|a_0}$, for a fixed $a_0$, a trivial predictive estimator, leading to the squared radius or diameter of $A$ being used there, instead of the smaller bounds on $I(A;Y)$ we have here using the risk of predictive estimators. 

To compare with \citep{Yang+},  the $Q_Y$ there is a uniform mixture of distributions on an $\epsilon'$ net producing the bound on $D(P_{Y|a}\|Q_Y)$ of $n (\epsilon')^2/(2\sigma^2)+ \log N(\epsilon',A)$, using the optimizing $\epsilon'$.  Hence for $\epsilon$ smaller than $\epsilon'$ (but of the same order), one also has a relationship between the $\epsilon$ entropy and the minimax risk.  The variant we use here avoids the need for two choices of $\epsilon$ and provides explicit comparison between metric entropy and any available predictive risk bound.

The $I(A;Y)$ is not more than the information capacity $C_{A}$
%$= \max_{P_a} \min_Q \int  D(P_{Y|a}\|Q_Y) dP_a$, 
of the Gaussian noise channel with inputs restricted to $A$, which is further upper bounded by $\max_{a \in A} D(P_{Y|a}\|Q_Y)$ for any particular choice of $Q_Y$.  So a way to say the results is that, for $\epsilon^2 >8 r_n(A)$, the metric entropy is less than $2C_A$, which in turn is controlled by the predictive risk.  The factor of two is replaced by a number near $1$ if the $\epsilon^2$ is enough greater than the risk $r_n(A)$. 
%\end{enumerate}
%\end{remark}

%\par
%\vspace{0.2 cm}

Next we give an analogous statement for function estimation with random input design.  Let $P$ be any probability measure on the input domain $\calX$ of a real-valued function.  Let $\calF$ be a class of functions $f$ mapping from $\calX$ to an interval $[-B,B]$.  Let $Y_i = f(X_i)+\epsilon_i$ where $\epsilon_i$ are independent normal random variables with mean $0$ and variance $\sigma^2$ for $i=1,2,\ldots,n$. The choice of normal errors is facilitate the demonstration of relationship between metric entropy and certain other measures of complexity, such as Rademacher complexity.  Implications for other noise distributions follow as consequences.

Estimators $\hat f_m$ are based on the samples $((X_i,Y_i)\,:\, i = 1,\ldots, m)$ for $m \le n$. Let $r_m(\calF)$ be a bound on the maximum of the risk $E[\|f-\hat f_m\|^2]$ for $f$ in $\calF$, where $\|f\!-\!g\|$ is the $L_2(P)$ distance. %where the best such is the minimax risk (based on samples of that size). 
Let $n \, r_n^* (\cal F)$ be the cumulative risk $\sum_{m=0}^{n-1} r_m(\calF)$. It may be interpreted as the cumulative risk for the sequence of predictive estimators $\hat f_{m}(X_{m+1})$ of $f(X_{m+1})$. Again $r_m(\calF)$ may depend on the error variance $\sigma^2$ and the range $B$ of the target function. It is familiar that in high-dimensional settings $r_n^*$ and $r_n$ are typically of the same order. Now let $\log N(\epsilon,\calF)$ be the metric entropy with respect to the $L_2(P)$ distance.

%\par{\vspace{0.1cm}
%\noindent {\bf Lemma 2:}

\begin{lemma} \label{lmm:riskmetric2}
For all $\epsilon > 0$,
$$r_n(\calF) \ge \frac{\epsilon^2}{4} \left( 1- \frac{n \,r_n^*(\calF)/(2\sigma^2) + \log 2}{\log N(\epsilon,\calF)}\right).$$
%Equivalently, for $\epsilon^2 > 4 r_n(\calF)$,
%$$\log N(\epsilon,\calF) \le \frac{n \, r_n^*(\calF)/(2\sigma^2) + \log 2} {1 - 4 r_n(\calF)/\epsilon^2}.$$
In particular, for $\epsilon^2 \ge 8 r_n(\calF)$,
$$\log N(\epsilon,\calF) \le n \, r_n^*(\calF )/\sigma^2 + 2 \log 2.
$$
\end{lemma}
%{\bf Proof}: 
\begin{proof} Similar to Lemma \ref{lmm:riskmetric}.  Here the mutual information, now $I(\calF, Z)$, has the role of the full data played by $Z= Z^n=(Z_1,Z_2,\ldots,Z_n)$ where $Z_m=(X_m,Y_m)$ has the inputs as well as the responses.  The $P_{Z^n|f}$ makes these independent with $P_{X_m}=P$ and $P_{Y_m|X_m,f}$ normal with mean $f(X_m)$, whereas the $Q_{Z^n}$ uses $Q_{Y_m|X_m,Z^{m-1}}$ normal with mean $\hat f_{m-1}(X_m)$ and matching variance $\sigma^2$. Then $I(\cal F,Z)$ is not more than $\max_{f \in \calF} D(P_{Z^n|f}||Q_{Z^n})$ where the $D(P_{Z^n|f}||Q_{Z^n})$ is $\sum_{m=1}^n E(f(X_m)\!-\!\hat f_{m-1}(X_m))^2/2\sigma^2$, not more than $n \,r_n^*(\calF)/2\sigma^2$. 
%This completes the proof of Lemma 2.
\end{proof}

%\par\vspace{0.2cm}
%\noindent {\bf Remarks:}
\begin{remark}
\indent
\begin{enumerate}
\item As a device for bounding the metric entropy $\log N(\epsilon,\calF)$, one has freedoms in the choice of $n$ and $\sigma^2$.
\item In statistical learning risk bounds for certain function classes take the form $r_0(\calF) \le (B+\sigma) C_\calF$
and for $m \ge 1$
\begin{equation} \label{eq:riskform}
r_m (\calF) = (B+\sigma) \frac{C_\calF}{\sqrt{m}}.
\end{equation}
This holds for least squares estimators, where $B$ controls the range of the target function, $\sigma^2$ is the variance of the Gaussian noise (or more generally a sub-Gaussian constant) and $C_\calF$ depends on the Gaussian complexity and the Rademacher complexity of $\calF$.  Bounds of this type will be discussed further in the next section.
\end{enumerate}
\end{remark}

When \eqref{eq:riskform} holds the cumulative risk has the sum of $1/\sqrt{m}$ which is near $2 \sqrt{n}$, and indeed $n r^*_n(\calF) \le 2 n r_n (\calF)$.  Whence the $r_n^*(\calF)$ may be replaced by the bound $2 r_n(\calF)$ in the above Lemma \ref{lmm:riskmetric2}.  Likewise for any polynomial rate of risk $(1/n)^{\alpha}$ with $0< \alpha < 1$ we may replace the $r_n^*(\calF)$ with $(1/\alpha) r_n (\calF)$.

Let's examine the matter with additional precision %in the $\alpha = 1/2$ case 
assuming \eqref{eq:riskform} holds.
For $m \ge 1$, the integral $\int_{m-0.5}^{m+0.5} x^{-1/2} dx$ equals $2[\sqrt{m+0.5}- \sqrt{m-0.5}\,]$, which is at least $1/\sqrt{m}$. So by integral comparison one finds for $n \ge 1$ that $\sum_{m=0}^{n-1} 1/\max\{\sqrt{m},1\}$ is not more than $1+\int_{0.5}^{n-0.5} x^{-1/2} dx$ which is $1+ 2 \sqrt{n-0.5} - 2/\sqrt{2}$, not more than 
$2\sqrt{n-0.5}$. The slack between $\sqrt{n-0.5}$ and $\sqrt{n}$ allows us to realize the following for all real $n\ge 1$ and $\sigma > 0$.
%we have the following: 
%{\bf Corollary:}  Suppose (1) holds as a bound on the risk of function estimation in $\calF$. 
For all $\epsilon^2 \ge 8 (B +\sigma^2)C_{\calF}/\sqrt{n}$,
$$\log N(\epsilon,\calF) \le \frac{2 \sqrt{n} (B+\sigma) C_{\calF}}{\sigma^2} + 2 \log 2.$$
In particular, setting $\sqrt{n} = 8 (B +\sigma)C_{\calF}/\epsilon^2$, we have
$$\log N(\epsilon,\calF) \le \frac{16 (B+\sigma)^2 C_{\calF}^2} {\epsilon^2 \sigma^2} + 2 \log 2.$$
Using Lemma \ref{lmm:riskmetric2} sets $\epsilon^2$ not larger than $8 (B +\sigma)C_{\calF}$ so that $n$ is at least $1$. This entropy bound holds trivally for $\epsilon^2$ at least $(B+\sigma)C_{\calF}$, as then the optimal $\epsilon$ cover entails use of just one point, per the bound on the no-data risk $r_0$, whence $N(\epsilon,\calF)=1$.  So this metric entropy bound holds for all $\epsilon > 0$. Taking the limit for large $\sigma$ we obtain the following.

%\par\vspace{0.2cm}
%\noindent {\bf Corollary 2:}
\begin{corollary} \label{cor:entropy}
For all $\epsilon > 0$, the metric entropy of a function class $\calF$ with risk bounds of the form \eqref{eq:riskform} satisfies
$$\log N(\epsilon,\calF) \le \frac{16 C_{\calF}^2} {\epsilon^2} + 2 \log 2.$$
\end{corollary}

\section{Statistical Risk} \label{sec:risk}

In this section, we show how the Rademacher and Gaussian complexities can be used to bound the statistical risk of a constrained least squares estimator over $ \calF_{L, V} $.

%\noindent {\bf Theorem:}
\begin{theorem} \label{thm:mainrisk}
Suppose $ \calF_{L, V, B} $ is the collection of all functions in $ \calF_{L, V} $ bounded by $ B $. Suppose $ f^* $ also belongs to $ \calF_{L, V, B} $. Then the constrained least squares estimator $ \hat f $ over $ \calF_{L, V, B} $ satisfies
\begin{equation} \label{eq:rate}
E[\|\hat f - f^*\|^2] \leq 2V(\sigma+4B)\sqrt{\frac{2(L\log(2)+\log (2d))}{n}}.
\end{equation}
\end{theorem}
\begin{remark}
The rate in \eqref{eq:rate} also holds for other %symmetric 
error distributions with sub-Gaussian tails. 
\end{remark}
\begin{remark}
Since the risk bound \eqref{eq:rate} has the form \eqref{eq:riskform},  a corollary of Theorem \ref{thm:mainrisk} and Corollary \ref{cor:entropy} is that the metric entropy of $\calF_{L,V}$ is not more than a constant multiple of $ V^2 (L \log 2 + \log(2d)) /\epsilon^2 + 2\log 2 $ for all $ \epsilon > 0 $.
\end{remark}
\begin{proof}
We can control the risk of the least squares estimator by bounding the expected supremum of the averages of
$$
L(Y'_i,f(X'_i)) - L(Y_i,f(X_i)),
$$
on independent future data $ (X'_i, Y'_i) $, in which the $X'_i$ have the same distribution as the $X_i$. Here we define $L(Y_i,f(X_i) = (Y_i-f(X_i))^2 - (Y_i-f^*(X_i))^2 $, which includes subtraction of the square of the noise $ \varepsilon_i = Y_i-f^*(X_i) $. 
Moreover, for the future $Y'_i$, there is no harm to let it be noise free $Y'_i = f^*(X'_i)$. Then the empirical discrepancy cleanly separates into the sum of a Gaussian and Rademacher complexity (normalized by $ 1/n $), viz., for any $ \tilde f $,
\begin{align*}
& \frac{1}{n}\sum_{i=1}^nE[L(Y'_i,\tilde f(X'_i)) - L(Y_i, \tilde f(X_i))] \\ & \leq 2E[\sup_{f\in\calF_{L, V, B}}\frac{1}{n}\sum_{i=1}^n \varepsilon_i (f(X_i)-f^*(X_i))] \\ & \qquad + E[\sup_{f\in\calF_{L, V, B}}\frac{1}{n}\sum_{i=1}^n \xi_i ((f(X'_i)-f^*(X'_i))^2-(f(X_i)-f^*(X_i))^2)] \\
& \leq 2E[\sup_{f\in\calF_{L, V, B}}\frac{1}{n}\sum_{i=1}^n \varepsilon_i f(X_i)] + 2E[\sup_{f\in\calF_{L, V, B}}\frac{1}{n}\sum_{i=1}^n \xi_i (f(X_i)-f^*(X_i))^2].
\end{align*}
where the expectations are with respect to Gaussian $ \varepsilon_i $, Rademacher $ \xi_i $, $ X_i $, and $ X'_i $. In particular, for the least squares estimator $ \hat f $ (minimizing $ f \mapsto \frac{1}{n}\sum_{i=1}^nL(Y_i,f(X_i)) $ for $ f \in \calF_{L, V, B} $), 
\begin{align*}
\frac{1}{n}\sum_{i=1}^nE[L(Y'_i,\hat f(X'_i)) - L(Y_i, f^*(X_i))] & \leq \frac{1}{n}\sum_{i=1}^nE[L(Y'_i,\hat f(X'_i)) - L(Y_i, \hat f(X_i))] \\ & \leq 2E[\sup_{f\in\calF_{L, V, B}}\frac{1}{n}\sum_{i=1}^n \varepsilon_i f(X_i)] \\ & \qquad + 2E[\sup_{f\in\calF_{L, V, B}}\frac{1}{n}\sum_{i=1}^n \xi_i (f(X_i)-f^*(X_i))^2],
\end{align*}
where the first inequality above arises from the fact that $ \hat f $ has smaller empirical risk than $ f^* $. Then, using $ \frac{1}{n}\sum_{i=1}^nE[L(Y'_i,\hat f(X'_i))] = E[\|\hat f - f^*\|^2] $ and combining the Gaussian and Rademacher complexity upper bounds from Corollary \ref{cor:complexity} (with minor modifications to handle squares of functions from $ \calF_{L, V} $ bounded by $ B $), we have
$$
E[\|\hat f - f^*\|^2] \leq 2V\sigma \sqrt{\frac{2(L\log(2)+\log (2d))}{n}} + 8VB\sqrt{\frac{2(L\log(2)+\log (2d))}{n}},
$$
which combine to be $ 2V(\sigma+4B)\sqrt{\frac{2(L\log(2)+\log (2d))}{n}} $. This completes the proof.
%\noindent
%{\bf Remark:} Let's consider tuning the constant. Presume $0 < \rho < 1$. If we set $\sqrt{n} = (4/\rho) (B +\sigma)C_{\calF}/\epsilon^2$ we have $4r_n(\calF)/\epsilon^2 = \rho$, so the metric entropy bound becomes $[n r_n (\calF )/\sigma^2 + \log 2]/[1-\rho]$ for which the leading term 
%is $4 (B+ \sigma)C_{\calF}/[\rho(1-\rho)]$.  It is indeed best at $\rho = 1/2$ which is the choice we made.
\end{proof}
\subsection{Adaptive estimation}
Flexible regression models are built by combining simple functional forms, which here consist of repeated compositions and linear transformations of nonlinear functions. In fitting such models to data in a training sample, there is a role for empirical performance criteria such as penalized squared error in selecting components of the function from a given library of candidate terms. With suitable penalty, optimizing the criterion adapts the total weights $(W_1, \dots, W_L)$ of combination or the number $ d_{\ell} $ of units $ z_{j_{\ell}} $ as well as the subset of which units to include in each layer. 
%The aim is to produce function estimates which accurately predict responses for new input values with the same distribution as the sample. 
In practice, one does not know the ``true'' $ V(f^*) $ for the regression function $ f^* $, which makes it difficult to select an upper bound $ V $ on $ V(f) $ for functions $ f $ in $ \calF_{L, V} $. In fact, $ f^* $ may not even be equal to a deep neural network and therefore empirical risk minimization over a finite covering of $ \calF_{L, V} $ is inconceivable unless the model is well-specified.
%The previously stated multi-hidden-layer-network risk bound \prettyref{eq:riskmain} is derived from nonadaptive estimators. 

Motivated by the previous concerns, an important question is whether the same rate in \eqref{eq:rate}, derived from nonadaptive estimators, is available from an adaptive risk bound (which allows for a more data-dependent and agnostic criterion for fits of $ f^* $) for estimators that minimize a penalized empirical risk over a finite $\epsilon_n$-covering $ \calU_{\calF_{L, \infty}}(\epsilon_n) $ of $ \calF_{L, \infty} $, where $ \calF_{L, \infty} = \bigcup_{V > 0}\calF_{L, V} $. That is, are adaptive risk bounds available for estimators $ \hat{f} $ with penalty defined through the ``smallest" variation $ V(f^*) $ among all representations of a network $ f^* $, i.e.,
$$
\frac{1}{n}\sum_{i=1}^n(Y_i - \hat{f}(X_i))^2 + V(\hat{f})\lambda_n \leq \inf_{f\in \calU_{\calF_{L, \infty}}(\epsilon_n)}\left\{ \frac{1}{n}\sum_{i=1}^n(Y_i - f(X_i))^2 + V(f)\lambda_n \right\},
$$
where $ \lambda_n \asymp \left(\frac{L+\log(d)}{n}\right)^{1/2} $ (whose choice is inspired by \eqref{eq:rate})? Indeed, it can be shown using the metric entropy calculation from Corollary \ref{cor:entropy} and techniques from \citep{Klusowski+} that
\begin{equation} \label{eq:adaptiverisk}
E[\|f^*-\hat{f}\|^2] \leq \inf_{f \in\calF_{L, \infty}}\left\{ \|f - f^*\|^2 + V(f)\lambda_n \right\},
\end{equation}
for such penalized estimation schemes, which cleanly expresses the approximation and model complexity tradeoff.

Penalties of similar flavor have already been established for single hidden layer networks, corresponding to $ L = 2 $. For example, using approximation results from \citep{Barron1991}, it was shown in \citep{Barron1994} that \eqref{eq:adaptiverisk} holds for a penalty $ \lambda_n \|W_1\|_1 $ with $ \lambda_n \asymp (d(\log(n/d))/n)^{1/2} $. Furthermore, recent work by \citep{E+} by we has shown that one can additionally control the size of the hidden layer parameters $ w_{j_1,j_2} $ through a penalty $ \lambda_n \|W_1W_2\|_1 $ with $ \lambda_n \asymp ((\log(d))/n)^{1/2} $. Note that \citep[Theorem 4]{Barron+2018} shows that these risk bounds for the single hidden layer case are essentially optimal in the sense that the minimax risk is lower bounded by $ v^{1/4}((\log(d))/n)^{1/2} $.

\bibliographystyle{plainnat}
\bibliography{ref}

\begin{thebibliography}{19}
\providecommand{\natexlab}[1]{#1}
\providecommand{\url}[1]{\texttt{#1}}
\expandafter\ifx\csname urlstyle\endcsname\relax
  \providecommand{\doi}[1]{doi: #1}\else
  \providecommand{\doi}{doi: \begingroup \urlstyle{rm}\Url}\fi

\bibitem[Barron(1991)]{Barron1991}
Andrew~R Barron.
\newblock Complexity regularization with application to artificial neural
  networks.
\newblock In \emph{Nonparametric Functional Estimation and Related Topics},
  pages 561--576. Springer, 1991.

\bibitem[Barron(1994)]{Barron1994}
Andrew~R Barron.
\newblock Approximation and estimation bounds for artificial neural networks.
\newblock \emph{Machine Learning}, 14\penalty0 (1):\penalty0 115--133, 1994.

\bibitem[Barron and Barron(1988)]{Barron+Barron}
Andrew~R Barron and Roger~L Barron.
\newblock Statistical learning networks: A unifying view.
\newblock In \emph{Computing Science and Statistics: Proceedings of the 20th
  Symposium on the Interface, Reston, Virginia}, pages 192--203. American
  Statistical Association, Alexandria, Virginia, 1988.

\bibitem[Barron and Klusowski(2018)]{Barron+2018}
Andrew~R Barron and Jason~M Klusowski.
\newblock Approximation and estimation for high-dimensional deep learning
  networks.
\newblock \emph{arXiv preprint arXiv:1809.03090}, 2018.

\bibitem[Barron et~al.(1999)Barron, Birg{\'e}, and Massart]{Barron+1999}
Andrew~R Barron, Lucien Birg{\'e}, and Pascal Massart.
\newblock Risk bounds for model selection via penalization.
\newblock \emph{Probability theory and related fields}, 113\penalty0
  (3):\penalty0 301--413, 1999.

\bibitem[Barron et~al.(2008)Barron, Cohen, Dahmen, and DeVore]{Barron+2008}
Andrew~R Barron, Albert Cohen, Wolfgang Dahmen, and Ronald~A. DeVore.
\newblock Approximation and learning by greedy algorithms.
\newblock \emph{Ann. Statist.}, 36\penalty0 (1):\penalty0 64--94, 2008.
\newblock ISSN 0090-5364.
\newblock \doi{10.1214/009053607000000631}.
\newblock URL \url{http://dx.doi.org/10.1214/009053607000000631}.

\bibitem[Bartlett and Mendelson(2002)]{Bartlett+Mendelson}
Peter~L Bartlett and Shahar Mendelson.
\newblock Rademacher and gaussian complexities: Risk bounds and structural
  results.
\newblock \emph{Journal of Machine Learning Research}, 3\penalty0
  (Nov):\penalty0 463--482, 2002.

\bibitem[Bartlett et~al.(2017)Bartlett, Foster, and Telgarsky]{Bartlett+}
Peter~L Bartlett, Dylan~J Foster, and Matus~J Telgarsky.
\newblock Spectrally-normalized margin bounds for neural networks.
\newblock In \emph{Advances in Neural Information Processing Systems}, pages
  6240--6249, 2017.

\bibitem[Birg{\'e}(1983)]{Birge}
Lucien Birg{\'e}.
\newblock Approximation dans les espaces m{\'e}triques et th{\'e}orie de
  l'estimation.
\newblock \emph{Zeitschrift f{\"u}r Wahrscheinlichkeitstheorie und verwandte
  Gebiete}, 65\penalty0 (2):\penalty0 181--237, 1983.

\bibitem[E et~al.(2018)E, Ma, and Wu]{E+}
Weinan E, Chao Ma, and Lei Wu.
\newblock A priori estimates of the generalization error for two-layer neural
  networks.
\newblock \emph{arXiv preprint arXiv:1810.06397}, 2018.

\bibitem[Golowich et~al.(2018)Golowich, Rakhlin, and Shamir]{Golowich+}
Noah Golowich, Alexander Rakhlin, and Ohad Shamir.
\newblock Size-independent sample complexity of neural networks.
\newblock In S\'ebastien Bubeck, Vianney Perchet, and Philippe Rigollet,
  editors, \emph{Proceedings of the 31st Conference On Learning Theory},
  volume~75 of \emph{Proceedings of Machine Learning Research}, pages 297--299.
  PMLR, 06--09 Jul 2018.
\newblock URL \url{http://proceedings.mlr.press/v75/golowich18a.html}.

\bibitem[Huang et~al.(2008)Huang, Cheang, and Barron]{Huang+}
Cong Huang, G.~LH Cheang, and Andrew~R Barron.
\newblock Risk of penalized least squares, greedy selection and $ \ell_1 $
  penalization for flexible function libraries.
\newblock \emph{Yale University, Department of Statistics Technical Report},
  2008.

\bibitem[Klusowski and Barron(2016)]{Klusowski+}
Jason~M Klusowski and Andrew~R Barron.
\newblock Risk bounds for high-dimensional ridge function combinations
  including neural networks.
\newblock \emph{arXiv preprint arXiv:1607.01434}, 2016.

\bibitem[Ledoux and Talagrand(1991)]{Ledoux+Talagrand}
Michel Ledoux and Michel Talagrand.
\newblock \emph{Probability in Banach Spaces: Isoperimetry and Processes},
  volume~23.
\newblock Springer Science \& Business Media, 1991.

\bibitem[Neyshabur et~al.(2015)Neyshabur, Tomioka, and Srebro]{Neyshabur+}
Behnam Neyshabur, Ryota Tomioka, and Nathan Srebro.
\newblock Norm-based capacity control in neural networks.
\newblock In \emph{Conference on Learning Theory}, pages 1376--1401, 2015.

\bibitem[Shalev-Shwartz and Ben-David(2014)]{Shalev-Shwartz+}
Shai Shalev-Shwartz and Shai Ben-David.
\newblock \emph{Understanding machine learning: From theory to algorithms}.
\newblock Cambridge University Press, 2014.

\bibitem[Tsybakov(2009)]{Tsybakov}
Alexandre~B. Tsybakov.
\newblock \emph{Introduction to nonparametric estimation}.
\newblock Springer Series in Statistics. Springer, New York, 2009.
\newblock ISBN 978-0-387-79051-0.
\newblock \doi{10.1007/b13794}.
\newblock URL \url{http://dx.doi.org/10.1007/b13794}.
\newblock Revised and extended from the 2004 French original, Translated by
  Vladimir Zaiats.

\bibitem[Vitale(2000)]{Vitale}
Richard Vitale.
\newblock Some comparisons for gaussian processes.
\newblock \emph{Proceedings of the American Mathematical Society}, 128\penalty0
  (10):\penalty0 3043--3046, 2000.

\bibitem[Yang and Barron(1999)]{Yang+}
Yuhong Yang and Andrew~R Barron.
\newblock Information-theoretic determination of minimax rates of convergence.
\newblock \emph{Ann. Statist.}, 27\penalty0 (5):\penalty0 1564--1599, 1999.
\newblock ISSN 0090-5364.
\newblock \doi{10.1214/aos/1017939142}.
\newblock URL \url{http://dx.doi.org/10.1214/aos/1017939142}.

\end{thebibliography}

%This is where the content of your paper goes.
%\begin{itemize}
%  \item Limit the main text (not counting references and appendices) to 12 PMLR-formatted pages, using this template.
%  \item Include, either in the main text or the appendices, all details, proofs and derivations required to substantiate the results.
%  \item Include {\em in the main text} enough details, including proof details, to convince the reviewers of the contribution, novelty and significance of the submissions.
%  \item Do not include author names (this is done automatically), and to the extent possible, avoid directly identifying the authors.%    You should still include all relevant references, including your own, and any other relevant discussion, even if this might allow a reviewer to infer the author identities.
%\end{itemize}

% Acknowledgments---Will not appear in anonymized version
%\acks{We thank a bunch of people.}

%\bibliography{yourbibfile}

%\appendix

%\section{My Proof of Theorem 1}

%This is a boring technical proof.

%\section{My Proof of Theorem 2}

%This is a complete version of a proof sketched in the main text.

\end{document}